 \documentclass[authoryear,preprint,review,12pt]{elsarticle}

\usepackage{amssymb}
\usepackage{subcaption}
\usepackage{graphicx}
\usepackage{amsmath}
\usepackage{hyperref}
\usepackage{amssymb}
\usepackage{amsthm}
\usepackage{url}
\usepackage{geometry}
\usepackage{color}
\usepackage{booktabs}
\usepackage[ruled,vlined]{algorithm2e}

\definecolor{RED}{rgb}{1,0,0}
\definecolor{ORANGE}{rgb}{1,0.5,0}
\definecolor{BLUE}{rgb}{0,0,1}

\usepackage{pgf, tikz} 
\usetikzlibrary{arrows,automata,fit}
\usetikzlibrary{shapes}
\newcommand{\xx}{1}
\newcommand{\yy}{1}
\newcommand{\sag}[2]{\tikz{\node[shape=circle,draw,inner sep=1pt,minimum width = 0.6cm, fill=#1]{$Y_{#2}$};}}

\newcommand{\specialcell}[2][c]{\begin{tabular}[#1]{@{}c@{}}#2\end{tabular}}
\newtheorem{lemma}{Lemma}
\usepackage{makecell}
\usepackage{bm}
\usepackage{lscape}

\usepackage{verbatim}

\usepackage{algorithmic}
\ifpdf
\DeclareGraphicsExtensions{.eps,.pdf,.png,.jpg}
\else
\DeclareGraphicsExtensions{.eps}
\fi

\begin{document}

\begin{frontmatter}



\title{Global Sensitivity Analysis of Uncertain Parameters in Bayesian Networks}


\author{Rafael Ballester-Ripoll}
\author{Manuele Leonelli}

\affiliation{organization={School of Science and Technology, IE University},
            city={Madrid},
            country={Spain}}

\begin{abstract}
Traditionally, the sensitivity analysis of a Bayesian network studies the impact of individually modifying the entries of its conditional probability tables in a one-at-a-time (OAT) fashion. 
However, this approach fails to give a comprehensive account of each inputs' relevance, since simultaneous perturbations in two or more parameters often entail higher-order effects that cannot be captured by an OAT analysis. We propose to conduct global variance-based sensitivity analysis instead, whereby $n$ parameters are viewed as uncertain at once and their importance is assessed jointly. Our method works by encoding the uncertainties as $n$ additional variables of the network. To prevent the curse of dimensionality while adding these dimensions, we use low-rank tensor decomposition to break down the new potentials into smaller factors. Last, we apply the method of Sobol to the resulting network to obtain $n$ global sensitivity indices. Using a benchmark array of both expert-elicited and learned Bayesian networks, we demonstrate that the Sobol indices can significantly differ from the OAT indices, thus revealing the true influence of uncertain parameters and their interactions.
\end{abstract}




\end{frontmatter}


\section{Introduction}
\label{sec:introduction}

Amongst the many machine learning methods now available to practitioners, Bayesian networks (BNs) are undoubtedly one of the most widespread due to their interpretability, ease of modelling, and wide array of software available \citep{Koller2009,pearl1988probabilistic}. BNs consist of a directed acyclic graph expressing the dependence relationship between the variables of interest, and one-dimensional conditional probability distributions from which various inferential queries can be answered. There are now uncountable applications of such models in a wide array of domains \citep[e.g.][]{Bielza2014a,Drury2017,hosseini2020bayesian,kabir2019applications,marcot2019advances,Mclachlan2020}. One of the main strengths of BNs is that they cannot only be learned from data, but they can also be expert-elicited, both in the structure of its underlying graph and its associated probabilities \citep{barons2022balancing}.

A fundamental but sometimes overlooked aspect of any real-world BN analysis is the robustness of outputs of interest w.r.t. misspecification of the model probabilities. In the field of BNs, this type of investigation is usually referred to as \textit{parametric sensitivity analysis} \citep{Rohmer2020}, and its use  has recently become more popular \citep[e.g.][]{giles2023solving,jindal2022bayesian,ma2022bayesian,wang2023chinese}.

The functional relationship between outputs and the BN model probabilities has been extensively studied in the literature \citep{Castillo1997,Coupe2002,laskey1995sensitivity,rohmer2020sensitivity,van2007sensitivity}. However, as \citet{Rohmer2020} has noted, most approaches that have been studied fall under the umbrella of one-at-a-time (OAT) sensitivity analysis or one-way sensitivity analysis. OAT investigates the effect of local variations to one parameter only while keeping the rest fixed, thus providing only a partial picture of the system's robustness. Attempts to vary more than one parameter at a time, the so-called \textit{multi-way} sensitivity analysis approach \citep{Bolt2014,bolt2017balanced,Kjaerulff2000,Leonelli2017}, quickly become computationally infeasible for most real-world BNs \citep{Chan2004,uai2008,Leonelli2022}, although some recent technical advances that map BNs to Markov chains or tensor networks (TNs) have been exploited to speed up computations \citep{ballester2023yodo,salmani2021fine,salmani2023automatically,salmani2023finding}.

Outside BNs, the most common approach is the so-called global sensitivity analysis \citep{Saltelli2000,SRACCGST:08}, which provides a ‘‘global’’, instead of local, representation of how different factors jointly influence some function of the model’s output. As noticed by \citet{saltelli2021sensitivity}, global sensitivity is en route to becoming an integral part of mathematical modelling. Its tremendous potential benefits are yet to be fully realized, both for advancing data-driven modelling of human and natural systems and in supporting decision-making \citep{Razavi2021}. Furthermore, the latest trends in machine learning focus on models whose workings can be interpreted and explained (explainable artificial intelligence, or XAI), and sensitivity analysis plays a central role in this task \citep{linardatos2020explainable,van2022comparison}.

To our knowledge, there have been only three recent attempts to implement global sensitivity methods for BNs. \citet{LM:17} and \citet{zio2022bayesian} implemented Monte Carlo-based estimation routines for global sensitivity indices. However, their approach suffers from two significant drawbacks: it is infeasible for moderate-size networks since it is based on a very complex Monte Carlo simulation; their implementation is not freely available, and therefore, other practitioners cannot utilize the routines in their work. \citet{ballester2022computing} implemented global sensitivity analysis techniques in BNs and developed exact methods for their computation based on TNs. Still, all these methods address the so-called sensitivity to evidence problem \citep[e.g.][]{gomez2014sensitivity}, which identifies the most informative evidence that could be observed to improve the assessment of the output of interest. 

 In this paper, we develop algorithms for the computation of global sensitivity indexes for parametric (as opposed to evidential) sensitivity analyses in BNs. To this end, we leverage the duality between probabilistic graphical models and TNs \citep{RS:18}. By first representing BNs as Markov random fields (MRFs) and then as TNs, the approach utilizes the duality between these representations to facilitate complex computations. This transformation allows for the application of sophisticated tensor operations, which are critical for handling the high-dimensional structures inherent in BNs. Specifically, the Sobol index computation is achieved through element-wise product operations followed by tensor contraction, enabling a more comprehensive sensitivity analysis that accounts for interactions between parameters.

More specifically, the proposed method augments tensors in a novel way, embedding uncertainties directly into the tensor structures. This augmentation, while potentially leading to computational infeasibility due to the increased dimensionality, is managed by employing state-of-the-art tensor contraction and manipulation libraries. These tools, particularly tensor train decompositions \citep{Oseledets:11}, allow for efficient processing by breaking down tensors into smaller, more manageable components. This decomposition not only mitigates the computational burden but also preserves the integrity of the sensitivity analysis, making it feasible to handle large and complex networks.

 In short, we develop a novel algorithmic solution for the interpretability of these machine learning models, providing a significant advance in knowledge and an impactful tool for applied analyses in many areas of science. This is particularly critical given that often ad-hoc solutions are implemented which may lead to false conclusions as observed by \citet{Do2020} and \citet{Saltelli2019}. In this paper, we showcase the method (and the insights it provides in practice) using a real-world application on the consequences of the COVID-19 pandemic at the European level.

\section{Background}

\subsection{Bayesian Networks}

Consider a categorical random vector of interest $\boldsymbol{Y}=(Y_1,\dots,Y_p)$. A BN gives a graphical representation of the relationship between the variables $\boldsymbol{Y}$ using a directed acyclic graph and a factorization of the overall probability distribution $P(\boldsymbol{Y}=\boldsymbol{y})$, where $\bm{y}$ is an element of the sample space $\mathbb{Y}=\times_{i\in\{1,\dots,p\}}\mathbb{Y}_i$, in terms of simpler conditional distributions $P(Y_i=y_i | \boldsymbol{Y}_{\Pi_i}=\boldsymbol{y}_{\Pi_i})$, where $\boldsymbol{Y}_{\Pi_i}$ denotes the parents of $Y_i$. More formally, the overall factorization of the probability distribution induced by the BN can be written as:
\[
P(\bm{Y}=\bm{y})=\prod_{i=1}^p P(Y_i=y_i | \boldsymbol{Y}_{\Pi_i}=\boldsymbol{y}_{\Pi_i})
\]

Figure \ref{fig:bn1} gives an example of a simple BN. The nodes are the variables of the problems and the arcs represent direct dependence relationships. The lack of an arc formally represents conditional independence as formalized by the \textit{d-separation criterion} of \cite{Pearl2009}. BNs, giving a graphical representation of the structure of the mathematical model, are particularly suited for applied analyses since they can be interpreted and understood having little, if no, previous mathematical knowledge.

\begin{figure}
\centering
\begin{tikzpicture}
\node (1) at (0*\xx,1*\yy){\sag{white}{1}};
\node (3) at (1.5*\xx,0*\yy){\sag{white}{3}};
\node (2) at (3*\xx,1*\yy){\sag{white}{2}};
\draw[->, line width = 1.1pt] (1) -- (3);
\draw[->, line width = 1.1pt] (2) -- (3);
\end{tikzpicture}
\;\;\;\;\;\;\;\;
\begin{tabular}{cc}
\toprule
\multicolumn{2}{c}{$Y_1$}\\
    no & yes  \\
     \midrule
     0.60&0.40 \\
     \bottomrule 
     \multicolumn{2}{c}{}\\
     \toprule
      \multicolumn{2}{c}{$Y_2$}\\
    no & yes  \\
     \midrule
     0.70&0.30 \\
     \bottomrule 
\end{tabular}
\;\;\;\;\;\;\;\;
\begin{tabular}{cc|cc}
\toprule
&&\multicolumn{2}{c}{$Y_3$}\\
$Y_1$&$Y_2$  &   no& yes  \\
     \midrule
    no & no & 0.70 & 0.30  \\
    no &yes& 0.60 & 0.40 \\
     yes & no& 0.60 & 0.40 \\
     yes& yes & 0.05 & 0.95\\
     \bottomrule
\end{tabular}

\caption{An example of a DAG over three binary random variables $Y_1,Y_2,Y_3$ with the associated probability specifications $\boldsymbol{\theta}^0$.
\label{fig:bn1}}
\end{figure}
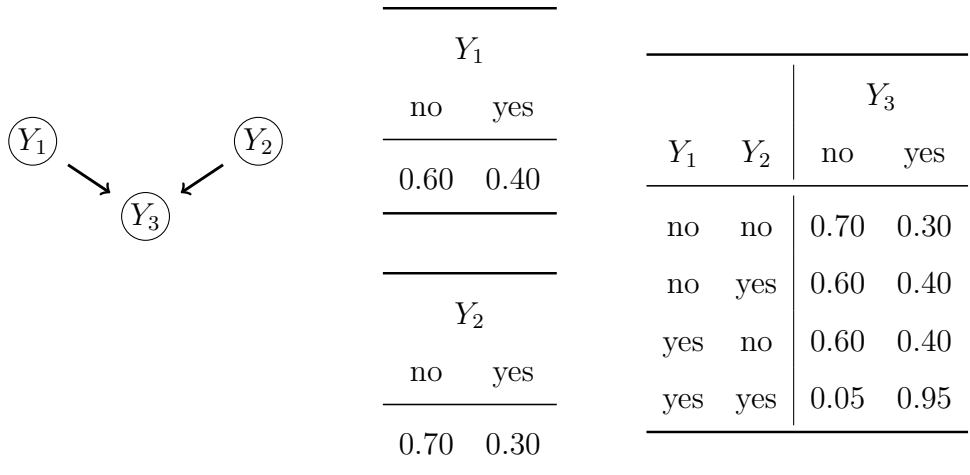

The conditional probabilities $P(Y_i=y_i | \boldsymbol{Y}_{\Pi_i}=\boldsymbol{y}_{\Pi_i})$ defining the BN model are usually denoted as $\boldsymbol{\theta}$ and, for categorical variables, form the so-called conditional probability tables (CPTs). The CPTs are either learned from data using machine learning algorithms \citep{scutari2019learns} or expert-elicited through probability elicitation exercises \citep{barons2022balancing}.  No matter the method used, we assume a value for the probabilities $\boldsymbol{\theta}$ has been chosen, which we refer to as the \textit{original value} and denote as $\boldsymbol{\theta}^0$. For our simple example, these are reported on the right-hand side of Figure \ref{fig:bn1}. Because of imprecisions during elicitation or low quality of observed data, it is imperative to check the effect of potential perturbations of $\boldsymbol{\theta}^0$ on the output of the BN model. This assessment is carried out via \textit{sensitivity analysis}.

\subsection{Parametric Sensitivity Analysis in BNs}
\label{sec:parametric_sensitivity_analysis}

We henceforth consider the conditional probabilities $\boldsymbol{\theta}$ defining the model as our objects of study. Let $Y_O$ be an output variable of interest 
and consider the probability of interest $f_{O}(\bm{\theta}) = P(Y_O=y_O)$ for some level $y_O$.
Seen as a function of a chosen vector of conditional probabilities $\boldsymbol{\theta}$, $f$ is called a \textit{sensitivity function} \citep{van2007sensitivity}.

Consider the BN in Figure \ref{fig:bn1} and suppose $Y_1$ and $Y_2$ describe two medical tests that could be either positive or negative, while $Y_3$ is the presence of a specific disease. Suppose our output of interest is $P(Y_3 = \textnormal{yes})$, the probability that an individual has the disease. In Figure \ref{fig:sens1}, we report the one-way sensitivity functions associated with the probability that a test is positive, constructed using the \texttt{bnmonitor} R package \citep{leonelli2023}. It can be clearly seen that individually, the two parameters have a very limited effect on the output probability, which changes very slowly as the parameter of interest is varied.

\begin{figure}
    \centering
    \includegraphics[scale=0.6]{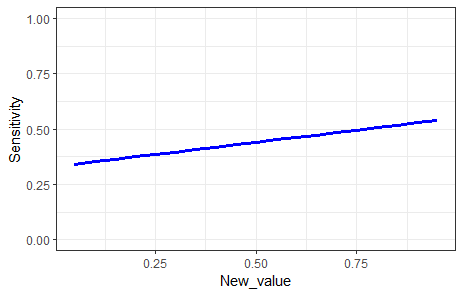}\;\;
    \includegraphics[scale=0.6]{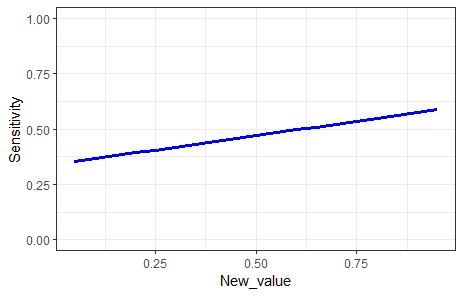}
    \caption{Sensitivity functions for the output probability $P(Y_3 = \textnormal{yes})$ as a function of $P(Y_1=\textnormal{yes})$ (left) and $P(Y_2=\textnormal{yes})$ (right).}
    \label{fig:sens1}
\end{figure}

A standard way to summarize one-way sensitivity functions is by considering the gradient at the original value of the considered parameter $\theta_i^0$. More formally, the $i$-th \textit{sensitivity value} is usually reported and formally defined as $|f'_{O}(\theta_i^0)|$. For the two sensitivity functions in Figure \ref{fig:sens1}, these values can be either computed manually using standard probability derivations or using an exact method such as YODO~\citep{ballester2022yodo}. They are very small in this case: 0.015 and 0.060, respectively.

\subsection{Proportional Covariation}

After the variation of  an input $\theta_i$ from its original value $\theta_i^0$ the sum-to-one condition of probabilities does not hold anymore. For this reason, parameters from the same conditional distribution are \textit{covaried} to respect this constraint. When variables are binary, this is automatic since one parameter must be equal to one minus the other. When variables are categorical with more than two levels, there are multiple methods to covary parameters  \citep{renooij2014co}. \textit{Proportional covariation} \citep{laskey1995sensitivity} is the gold-standard method since it is motivated by  a wide array of theoretical properties \citep{Chan2005,Leonelli2017,Leonelli2022}. When a parameter $\theta_j$ from the same conditional distribution of the varied parameter $\theta_i$ is proportionally covaried it equals:
\[
\theta_j(\theta_i) = \frac{1-\theta_i}{1-\theta^0_i}\theta^0_j.
\]
Under the assumption of proportional covariation, \citet{Castillo1997} and \citet{coupe2000sensitivity} proved that the sensitivity function equals the ratio of two linear functions:
\[
f(\theta_i)=\frac{c_0+c_i\theta_i}{d_0+d_i\theta_i},
\]
where $c_0,c_i,d_0,d_i\in\mathbb{R}$.

\subsection{Motivating Example: Limitation of OAT Methods}
\label{sec:limit}

\begin{figure}
    \centering
    \includegraphics[scale=0.7]{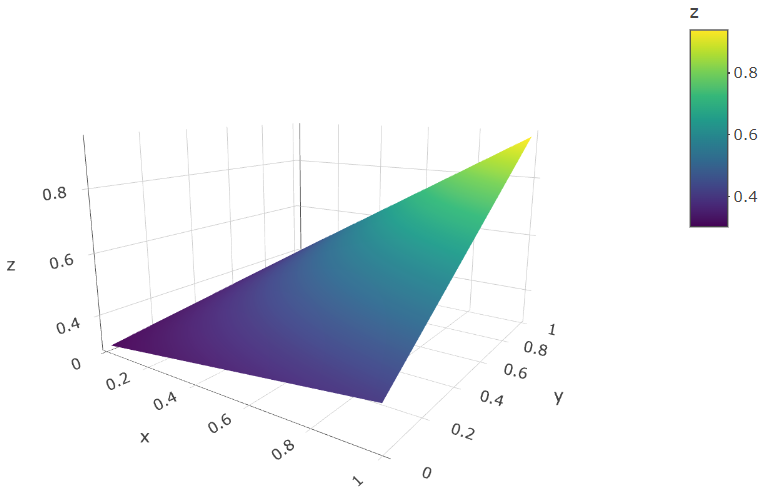}
    \caption{Two-way sensitivity function for $P(Y_3 = \textnormal{yes})$ as a function of $P(Y_1=\textnormal{yes})$ and $P(Y_2=\textnormal{yes})$. }
    \label{fig:enter-label}
\end{figure}

It is well known that OAT sensitivity methods only provide a partial and sometimes wrong picture of the influence of the input parameters \citep{Saltelli2019}. Our simple BN example in Figure \ref{fig:bn1} illustrates this. Consider now the probability of having the disease $P(Y_3=\textnormal{yes})$ as a function of the two tests being positive $P(Y_1=\textnormal{yes})$ and $P(Y_2=\textnormal{yes})$. Figure \ref{fig:enter-label} reports this probability as a function of both parameters. It can be clearly seen there is a steep increase in the probability of having the disease when the probability of both tests being positive increases. On the other hand, the two parameters individually have a minor impact on the output probability, as already noticed in Figure \ref{fig:sens1}. To draw a parallelism, this is similar to a two-way ANOVA when there is a strong interaction term between the two factors. Of course, this example is trivial and the joint effect of both parameters could have been already noticed by simply looking at the conditional probabilities in Figure \ref{fig:bn1}. Still, it illustrates the limitation of OAT analyses in BNs even in the simplest scenario.

In the context of BNs, multi-way sensitivity analyses where multiple parameters are varied simultaneosly have been studied. However, their computation and interpretation becomes challenging. We can notice that (i) sensitivity
functions cannot be visualized in high dimensions; (ii) the number of groups of parameters that could be varied grows exponentially; (iii) most critically, the associated measures are challenging to interpret, similar to higher-order interactions in standard statistical
models \citep[see e.g.][]{hayes2012cautions}.
For these reasons, we argue that a global approach is better suited to thoroughly investigate the effect of the parameters of a BN model.

\subsection{Sobol Indices}

Sobol indices represent one of the most widely used frameworks for global sensitivity of a function $f: \Omega\subseteq \mathbb{R}^n\rightarrow \mathbb{R}$ taking values in a $n$-dimensional rectangle $\times_{i\in\{1,\dots,k\}}\Omega_i$ \citep{Sobol:90}. Although there are many different types of indices, for our purposes we focus on two of the most common ones, namely the \emph{variance component} and the \emph{total index}.

Henceforth, consider a $n$-dimensional vector $\bm{\theta}=(\theta_1,\dots,\theta_n)$ of input parameters, whose uncertainty needs to be assessed. The \emph{variance component} associated to $\theta_i$ is 
	\begin{equation}
		S_i := \frac{\mathrm{Var}_i \left[ \textnormal{E}_{\setminus i}[f] \right] }{\mathrm{Var}[f]},
		\label{eq:variance_component}
	\end{equation}
	where $ \textnormal{E}_{\setminus i}$ is the expectation with respect to $\pmb{\theta}_{\setminus i}$ and $\mathrm{Var}_i$ is the variance with respect to $\theta_i$. The variance component quantifies how the average model output changes as $\theta_i$ is varied. In other words, it is an additive measure over $\theta_i$.
	
The \emph{total index} of the parameter $\theta_i$ swaps the roles of the expectation and the variance in Eq.~(\ref{eq:variance_component}):
\[
		S^T_i := \frac{ \textnormal{E}_{\setminus i} \left[ \mathrm{Var}_i[f] \right] }{\mathrm{Var}[f]}.
\]
The total index quantifies the average variability in the model output due to changes in $\theta_i$ only. Compared to the variance component, the total index is an overall measure, which includes additive effects but also interaction terms. Thus, $S_i \le S^T_i$. The two indices are equal if and only if  $f(\pmb{\theta}) = g(\theta_i) + h(\pmb{\theta}_{\setminus i})$, that is if $\theta_i$ can be separated from $f$ in an additive fashion. Among other uses, total indices are a popular tool for model simplification: whenever a model input $\theta_i$ has a small $S^T_i$, it can be safely fixed to any reasonable constant, thus decreasing the degrees of freedom.

\subsection{Tensor Train Decomposition}
\label{sec:tt_decomposition}

The tensor train low-rank decomposition~\citep{Oseledets:11}, or TT for short, factorizes a high-dimensional tensor into a linear network of tensors that are at most three-dimensional. If $\mathcal{T}$ is a tensor of shape $I_1 \times \dots \times I_N$, the TT format approximates it as
\begin{equation}
    \label{eq:tt}
    \mathcal{T}(i_1, \dots, i_N) \approx G_1(i_1) \cdots G_N(i_N)
\end{equation}
where $G_1(i_1)$ is a row vector, $G_2(i_2), \dots, G_{N-1}(i_{N-1})$ are matrices, and $G_N(i_N)$ is a column vector for any values $1 \le i_1 \le I_1, \dots, 1 \le i_N \le I_N$. 

We can use this format to break up CPTs into smaller tensors, thus decreasing their number of degrees of freedom. For example, a node $A$ with three parents $B, C, D$ will have a CPT $\Phi$ of 4 dimensions. Viewed as a tensor, it can be written in the TT format as
\begin{equation}
    \label{eq:tt_example}
    \Phi(y_A, y_B, y_C, y_D) \approx \sum_{r_1, r_2, r_3, r_4} G_1(y_A, r_1) G_2(r_1, y_B, r_2) G_3(r_2, y_C, r_3) G_4(r_3, y_D)
\end{equation}
where the $r_1, \dots, r_{N-1}$ are known as bond dimensions (or \emph{TT ranks}) and play the role of latent variables. The model reduction induced by the TT format comes at the expense of a variable compression error $\epsilon$. In any case, that error may be reduced by increasing $\pmb{r}$ as necessary, and $\epsilon$ can be even brought to zero when the tensor is exactly low-rank. 

Note that directly compressing a CPT as in Eq.~(\ref{eq:tt_example}) is rarely useful. A more common use of low-rank decompositions in this context is to apply truncated singular value decomposition (SVD) to break up intermediate factors during tensor contraction~\citep{graychan24} (which is equivalent to variable elimination in exact BN inference). In this paper, we will use the TT format to handle ``enriched'' CPTs that incorporate additional degrees of freedom due to parametric uncertainties in the BN. In practice, we will use the following variant of Eq.~(\ref{eq:tt}):
\begin{equation}
    \label{eq:ett}
    \mathcal{T}(i_1, \dots, i_N) \approx (G_1 \times_2 \mathbf{U}_1)(i_1) \cdots (G_N \times_2 \mathbf{U}_N)(i_N)
\end{equation}
where the $\times_2$ is the \emph{tensor-times-matrix} product along the second mode~\citep{KB:09} and the $\mathbf{U}_1, \dots, \mathbf{U}_N$ are learned matrices that provide a convenient extra level of compression when dimensions represent continuous variables (as is the case for our uncertainties). Although Eq. (\ref{eq:ett}) has sometimes been called \emph{extended tensor train} in the literature~\citep{SU:14}, we will simply refer to it as \emph{TT} in this paper.

\section{The Algorithm}
\label{sec:algorithm}

Our goal is to apply Sobol's method to address the limitations of OAT parametric sensitivity analyses in BNs illustrated in Section \ref{sec:limit}. For simplicity, we will refer to all probability tables in the model as \emph{CPTs} even when a node has no parents. Our algorithm takes a BN and a list of uncertain CPT entries $\bm\theta$ and produces a Sobol index for each $\theta_i$. It consists of four steps:

\begin{enumerate}
    \item We moralize the BN and interpret it as a tensor network.
    \item For each CPT that includes uncertain entries, we encode each uncertainty as an additional new parent (i.e., a new dimension of the CPT).
    \item To avoid exponential blow-up of the augmented CPTs, we compactly encode them using a low-rank TT decomposition. Within the broader network, we seamlessly replace the original affected CPTs for their TT versions. For computational efficiency reasons, we execute this step at the same time as the previous one.
    \item Using the method of Sobol, we extract sensitivity indices of the resulting TN for the set of added new variables. We use the algorithm of~\cite{ballester2022computing} with exact inference for this step.
\end{enumerate}

Next, we detail each step in depth.

\subsection{Moralization and Conversion to a Tensor Network}

Graph moralization is a standard step for casting a BN into a Markov random field (MRF) \citep[e.g.][]{Koller2009}. It involves two steps:
\begin{enumerate}
    \item Removing all arrow directions.
    \item Connecting (\emph{marrying}) each pair of parents together, so that each CPT for a node with $k$ parents becomes a clique of $(k+1)$ nodes with an associated $(k+1)$-dimensional \emph{potential}.
\end{enumerate}

Furthermore, in this paper we represent MRFs as tensor networks (TNs). The TN formulation is equivalent to that of MRFs~\citep{RS:18}, since each one is the dual graph of the other:

\begin{itemize}
    \item The nodes of an MRF become edges (or hyperedges, if the node has three or more neighbors) in the TN.
    \item The edges (network cliques are counted as hyperedges) of an MRF become nodes in the TN. The \emph{tensors} in the TN are simply the MRF potentials.
\end{itemize}

See Figure~\ref{fig:bn_mrf_tn1} for an example of moralization and conversion to TN by means of the dual graph.

\begin{figure}\centering
	\begin{subfigure}[b]{0.4\columnwidth}
		\includegraphics[width=1\columnwidth]{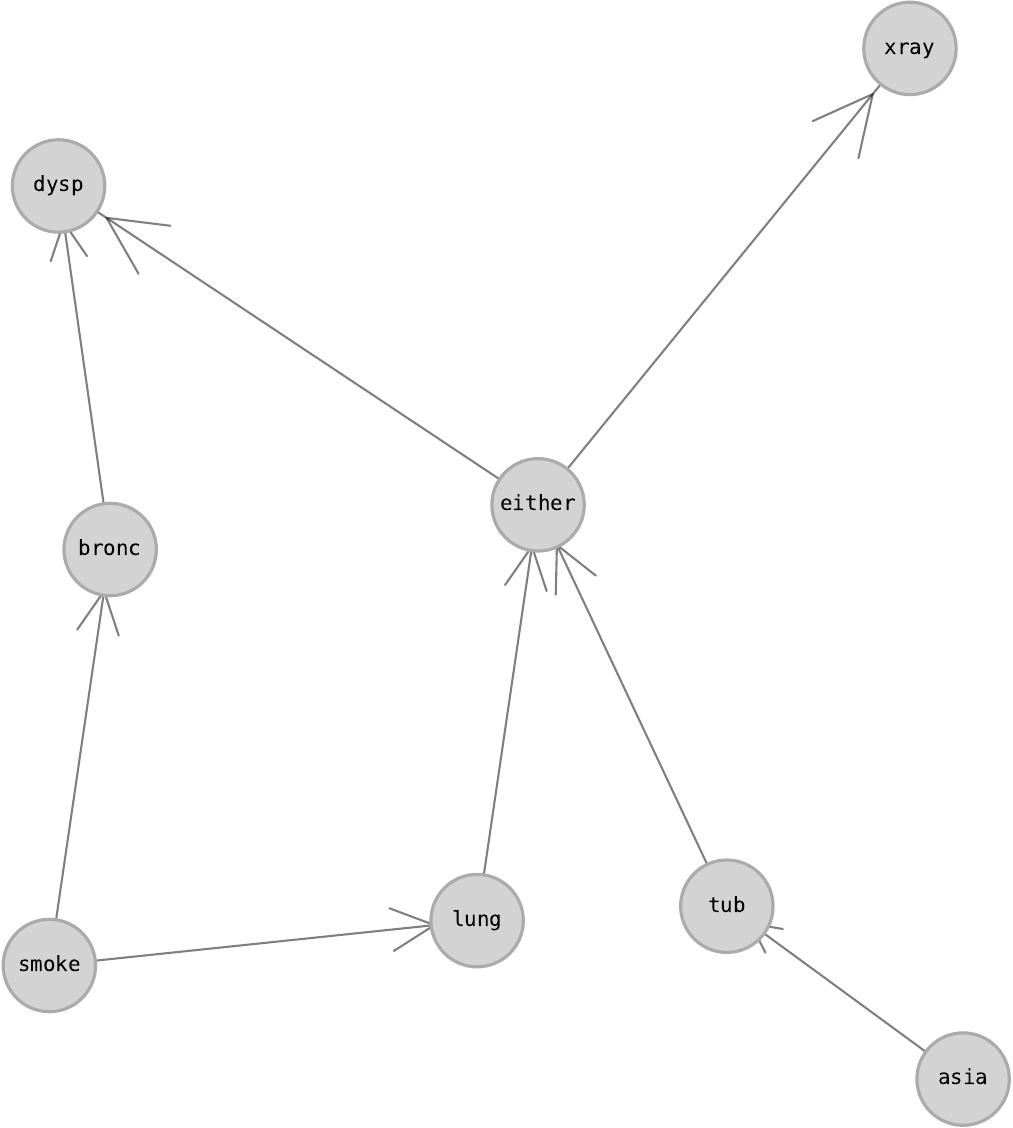}
		\caption{An input Bayesian network $\mathcal{B}$.}
	\end{subfigure}%
	\begin{subfigure}[b]{0.4\columnwidth}
		\includegraphics[width=1\columnwidth]{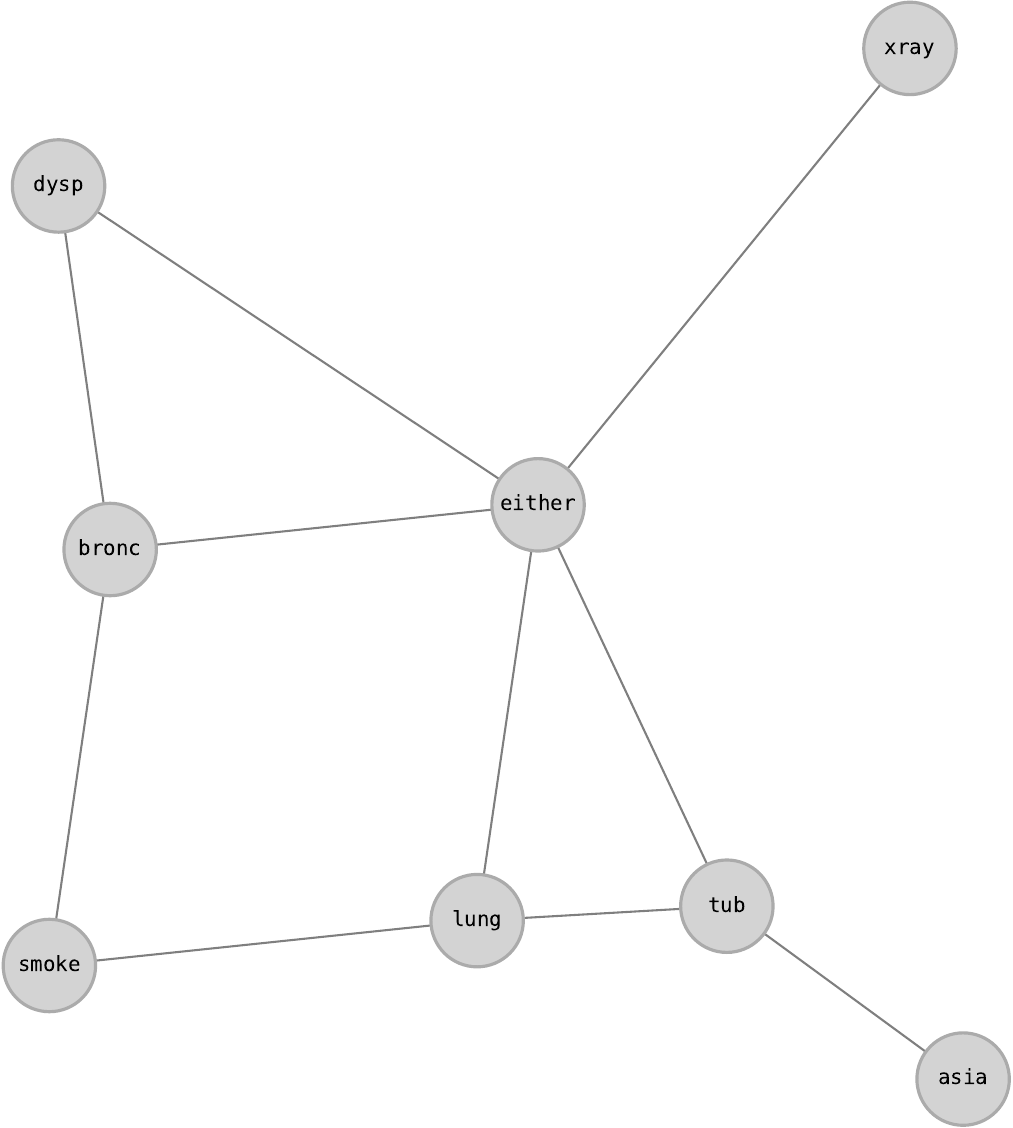}
		\caption{MRF $\mathcal{M}$ that results from moralizing $\mathcal{B}$.}
	\end{subfigure}
	\hfil
	\begin{subfigure}[b]{0.6\columnwidth}
		\includegraphics[width=1\columnwidth]{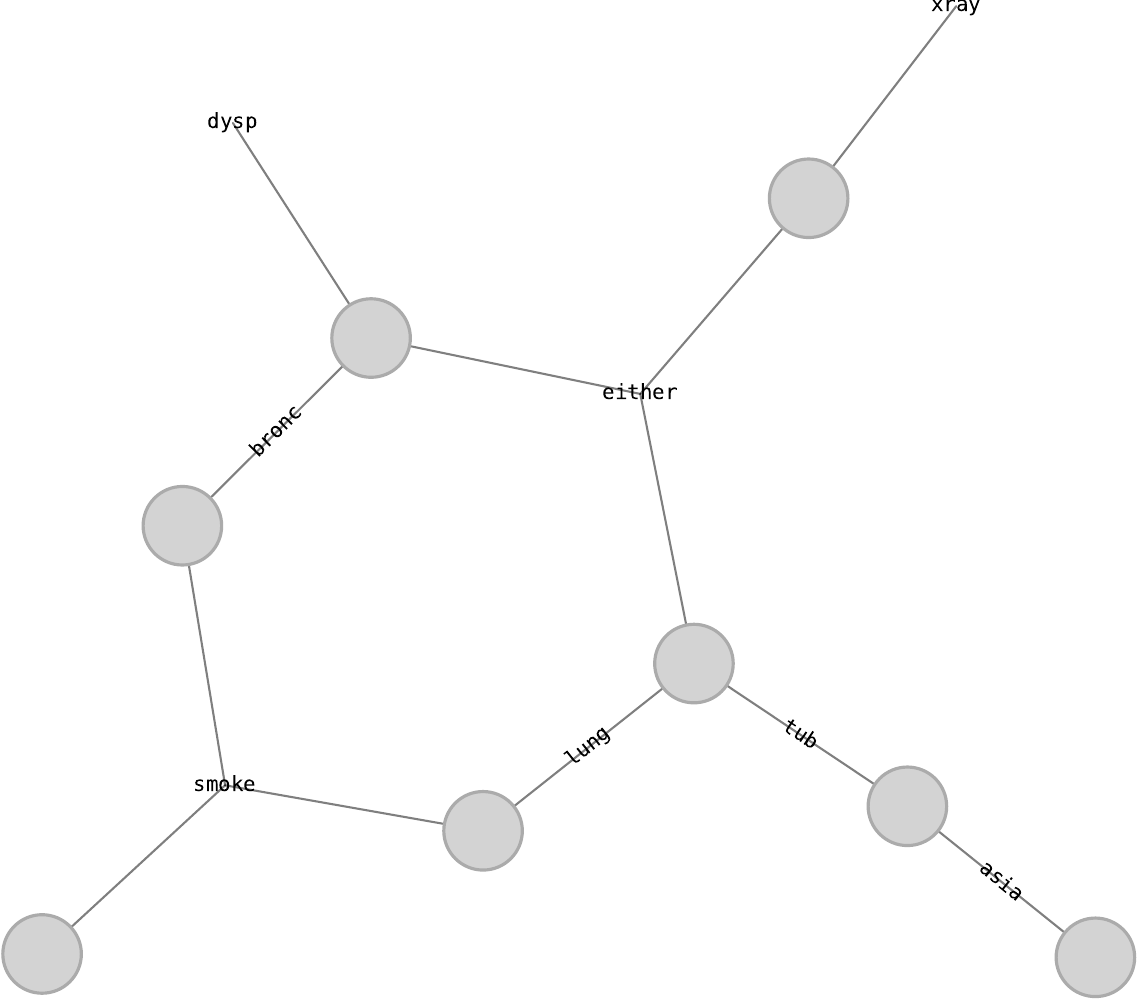}
		\caption{Tensor network representation (dual of (b)).}
	\end{subfigure}
	\caption{The \textsc{asia} network as a Bayesian network (a), Markov random field (b), and tensor network (c). Note that, if the cliques of (b) are regarded as hyperedges, then (c) is the dual graph of (b).}
	\label{fig:bn_mrf_tn1}
\end{figure}

The TN representation makes it easier to illustrate the ``surgical'' network manipulation operations we perform in later steps, since the low-rank decomposition we use (tensor train) is usually represented as a TN in the literature.

\subsection{Encoding Uncertainties}
\label{sec:encoding_uncertainties}

Let the CPT for variable $Y$ with parents $\boldsymbol{Y}_{\Pi}$ be denoted as $\Phi(Y=y, \boldsymbol{Y}_{\Pi}=\boldsymbol{y}_{\Pi})$. Let us assume we have $P$ uncertain entries in $\Phi$, with $y_1, \dots, y_P$ denoting the child and $\pmb{y}_1^{\mathbb{U}}, \dots, \pmb{y}_P^{\mathbb{U}}$ denoting the parent indices of those entries. As an example, if the uncertain entries in the CPT of a node $A$ are $\{P(A=1 | B=0, C=0), P(A=1 | B=1, C=1)$, then $y_1 = 1, y_2 = 1, \boldsymbol{y}_1^{\mathbb{U}} = (0, 0), \boldsymbol{y}_2^{\mathbb{U}} = (1, 1)$. Note that we do not support multiple uncertain entries for a single parent configuration, i.e. $\pmb{y}_a^{\mathbb{U}} \ne \pmb{y}_b^{\mathbb{U}}$ for all $a \ne b$. The reason is that the constraint $\sum_{y} \Phi(Y=y, \boldsymbol{Y}_{\Pi}=\boldsymbol{y}_{\Pi}) = 1$ imposes a relationship between these entries, making them correlated and hence unsuitable for a Sobol analysis.

Our goal is to build an ``augmented'' potential $\Phi'$ that includes new dimensions $x_1, \dots, x_P$ with $x_p \in [0, 1]$ for all $p$. Both $\Phi'$ and $\Phi$ are identical (and the $x_p$ dummy) everywhere except on the uncertain indices. On each uncertain entry $p$, we have that $\Phi'(\boldsymbol{x}, y_p, \boldsymbol{y}_p^{\mathbb{U}}) := x_p$, whereas entries of $\Phi'$ for other child values $Y \ne y_p$ are adjusted accordingly using proportional covariation~\citep{laskey1995sensitivity}. More compactly:
%
\[
\begin{aligned}
    \Phi'(\boldsymbol{X} = \boldsymbol{x}, Y = k, \boldsymbol{Y}_{\Pi} = \boldsymbol{y}_{\Pi}) := \\
    \begin{cases}\frac{1 - x_p}{1 - \Phi(y_p, \boldsymbol{y}_{\Pi})} \Phi(k, \boldsymbol{y}_{\Pi}) & \mbox{ if } \boldsymbol{y}_{\Pi} = \boldsymbol{y}_p^{\mathbb{U}} \mbox{ for some } p \in \{1, \dots, P\} \\
    \Phi(k, \boldsymbol{y}_{\Pi}) & \mbox{otherwise.}
    \end{cases}
\end{aligned}
\]

The new tensor $\Phi'$ has $P$ more dimensions than $\Phi$ and could become computationally intractable even at modest values of $P$. Fortunately, it can be succinctly represented in the TT format by virtue of the following lemma.

\begin{lemma}
    \label{lemma:rank_bound}
    The potential $\Phi'$ defined above has TT rank at most $\sqrt{|\Phi|} + P$, where $|\Phi|$ is the total number of elements of $\Phi$.
\end{lemma}

\begin{proof} Let us break down the augmented potential as  $\Phi' := \Phi + \mathcal{T}_1 + \dots \mathcal{T}_P$, where each $\mathcal{T}_p$ is a function defined as:
\begin{equation}
\mathcal{T}_p(\pmb{x}, y, \boldsymbol{y}_{\Pi}) :=
\begin{cases}\frac{1 - x_p}{1 - \Phi(y_p, \boldsymbol{y}_{\Pi})} \Phi(k, \boldsymbol{y}_{\Pi}) - \Phi(k, \boldsymbol{y}_{\Pi}) & \mbox{ if } \boldsymbol{y}_{\Pi} = \boldsymbol{y}_p^{\mathbb{U}} \\
    0 & \mbox{otherwise.}
\end{cases}
\label{eq:tp}
\end{equation}

Essentially, Eq.~(\ref{eq:tp}) replaces the original distribution $\Phi(Y = k)$ for $k = 1, \dots, |Y|$ by subtracting it and adding its proportionally covaried version instead; in other words, it ensures the new distribution still sums to 1 for all values of $x_p$ in $[0, 1]$.

Next, let us see that $\mathcal{T}_p$ has TT rank 1 for any $p = 1, \dots, P$. Its non-zero terms equal $\left( \frac{1 - x_p}{1 - \Phi(y_p, \boldsymbol{y}_{\Pi})} - 1 \right) \Phi(k, \boldsymbol{y}_{\Pi})$, which is a separable function in terms of $x_p$ and $k$. Therefore, $\mathcal{T}_p$ is a zero tensor everywhere except for one slice which contains a rank-1 matrix. Adding dummy dimensions (\emph{unsqueezing}) a tensor does not increase its TT rank, hence the result must have TT rank equal to 1.

Last, to obtain the stated bound, note that the TT rank of a tensor $\Phi$ is defined as the rank of its \emph{unfolding matrices}~\citep{Oseledets:11}. Such matrices have shape at most $\sqrt{|\Phi|} \times \sqrt{|\Phi|}$, hence their rank is bounded by $\sqrt{|\Phi|}$. Using the property that the rank of a summation of TT tensors is at most the sum of their individual ranks, the final TT rank of $\Phi' = \Phi + \sum_p \mathcal{T}_p$ is at most $\sqrt{|\Phi|} + P$.

\end{proof}

Since the proof of Lemma~\ref{lemma:rank_bound} is constructive, we turn it into a practical algorithm to assemble $\Phi'$. To do so, we discretize each $x_p$ into $I$ bins equally spaced in the interval $[0, 1]$. See Alg.~\ref{alg:main} for details.

\begin{algorithm}
    \small
	\caption{Given a conditional probability table of a node $Y$ with $N$ parents, which we view as an $(N+1)$-dimensional tensor, and a list of $P$ uncertain entries, build a $(P+N+1)$-dimensional TT. Operations starting with ``\texttt{tt\_}'' are performed in the TT compressed domain.}
	\label{alg:main}
	\begin{algorithmic}
		\STATE{\textbf{Input}: CPT $\Phi$, indices of the uncertain entries: $y_1, \dots, y_p$ for the child states and $\pmb{y}_1^{\mathbb{U}}, \dots, \pmb{y}_P^{\mathbb{U}}$ for its parents}
  		\STATE{\textbf{Output}: tensor train approximation of the enriched potential $\Phi'$}
        \FOR{$p = 1, \dots, P$}
                \STATE $\blacktriangleright$ $\mathcal{M} \leftarrow \texttt{zeros}(I, |Y_p|)$
                \FOR{$i = 1, \dots, I$}
                    \STATE $\blacktriangleright$ $x_p = (i-1)/(I-1)$ // Make $x_p$ move in $[0, 1]$ at regular intervals
                    \FOR{$j = 1, \dots, |Y_p|$}
                        \STATE $\blacktriangleright$ $\mathcal{M}(i, j) \leftarrow \left( \frac{1 - x_p}{1 - \Phi(y_p, \boldsymbol{y}_{\Pi})} - 1 \right) \Phi(k, \boldsymbol{y}_{\Pi})$
                    \ENDFOR
                \ENDFOR
                \STATE $\blacktriangleright$ $\mathcal{M} \leftarrow \texttt{tt\_compress}(\mathcal{M})$
                \STATE $\blacktriangleright$ $\mathcal{M} \leftarrow \texttt{tt\_unsqueeze}(\mathcal{M}, \{1, \dots, \hat{p}, \dots P\})$ // $\mathcal{M}$ gets dimensionality $P + 1$
                \STATE $\blacktriangleright$ $\mathcal{M} \leftarrow \texttt{tt\_tile}(\mathcal{M}, \{I\}_{i=1}^{p-1} \cup \{1\} \cup \{I\}_{i=p+1}^{P} \cup \{1\})$ // $\mathcal{M}$ gets size $I^P \times |Y_p|$
                \STATE $\blacktriangleright$ $\mathcal{T}_p \leftarrow \texttt{tt\_zeros}(\{I\}_{i=1}^P, |Y|, \{|Y| \mbox{ for } Y \in Y_{\Pi}\})$
                \STATE $\blacktriangleright$ $\mathcal{T}_p(..., \boldsymbol{y}_p^{\mathbb{U}}) \leftarrow \mathcal{M}$
        \ENDFOR
        \STATE $\blacktriangleright$ $\mathcal{T} \leftarrow \texttt{tt\_compress}(\Phi)$
		\RETURN $\texttt{tt\_sum}(\mathcal{T}, \mathcal{T}_1, \dots, \mathcal{T}_P)$
	\end{algorithmic}
\end{algorithm}

See Fig.~\ref{fig:bn_mrf_tn} for an illustration of the procedure using tensor network diagrams.
\begin{figure}\centering
	\begin{subfigure}[b]{0.3\columnwidth}
		\includegraphics[width=1\columnwidth]{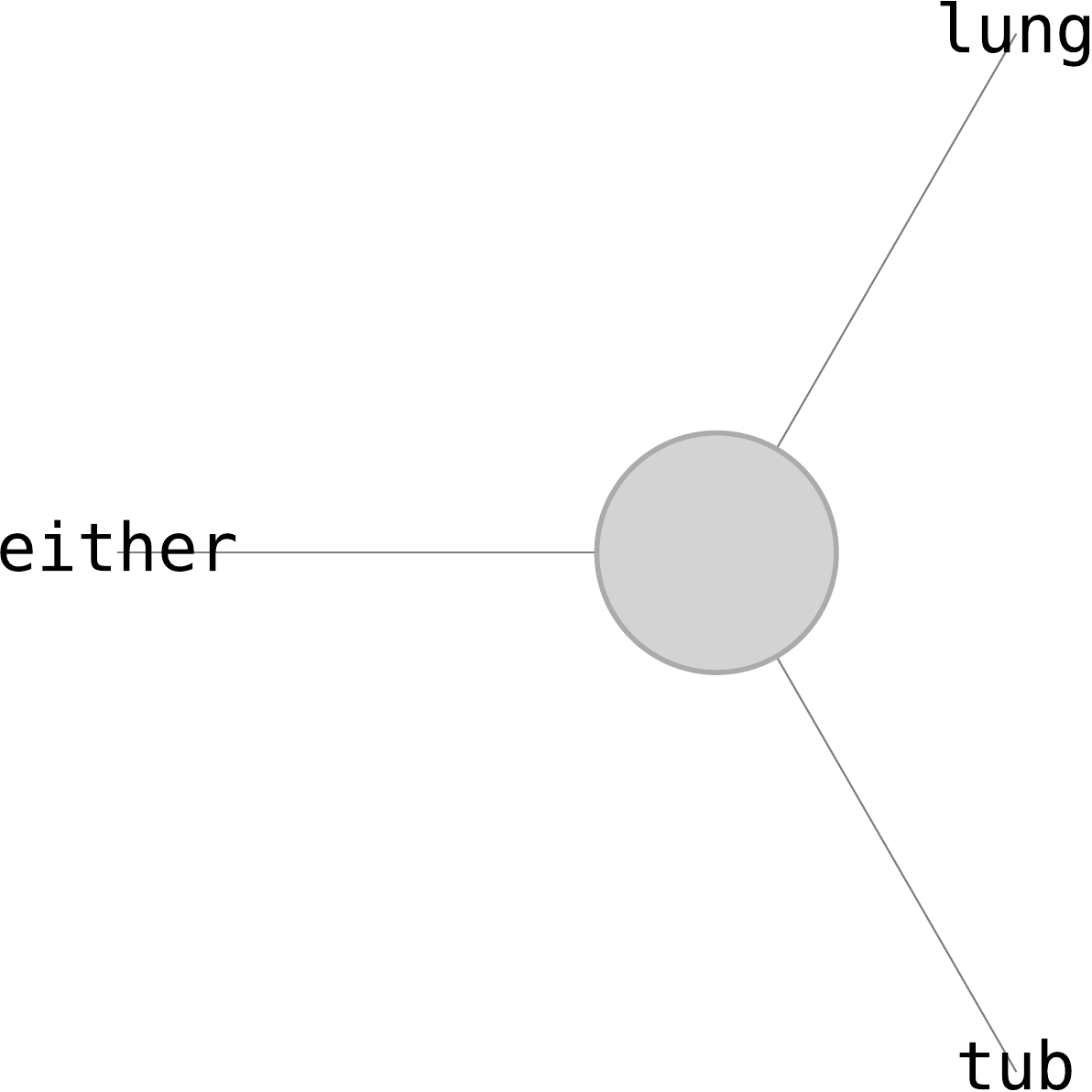}
		\caption{The CPT $\Phi$ for variable \emph{either} and its two parents, expressed as a tensor (node) in a TN.}
	\end{subfigure}
    \hfil
	\begin{subfigure}[b]{0.3\columnwidth}
		\includegraphics[width=1\columnwidth]{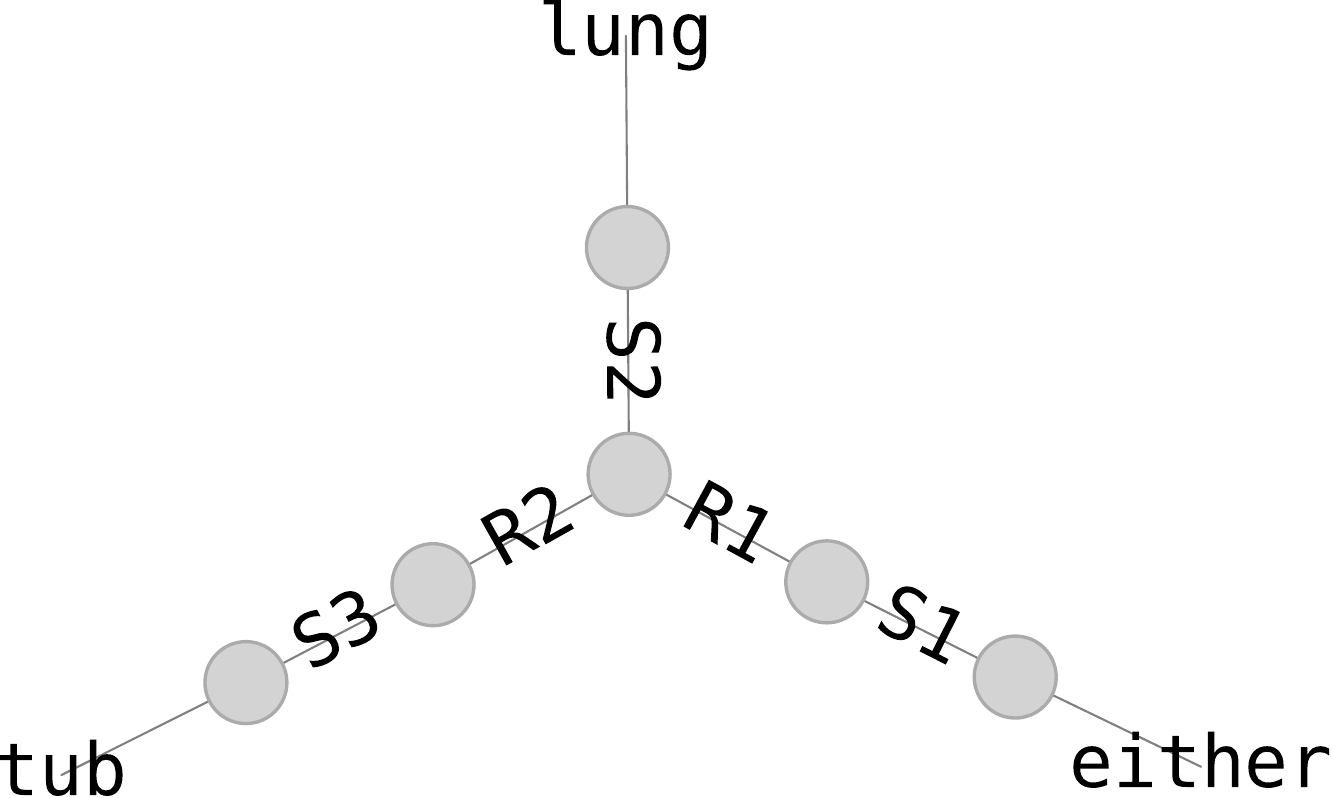}
		\caption{TT representation of $\Phi$, without any uncertainties.}
	\end{subfigure}
	\\
	\begin{subfigure}[b]{0.6\columnwidth}
		\includegraphics[width=1\columnwidth]{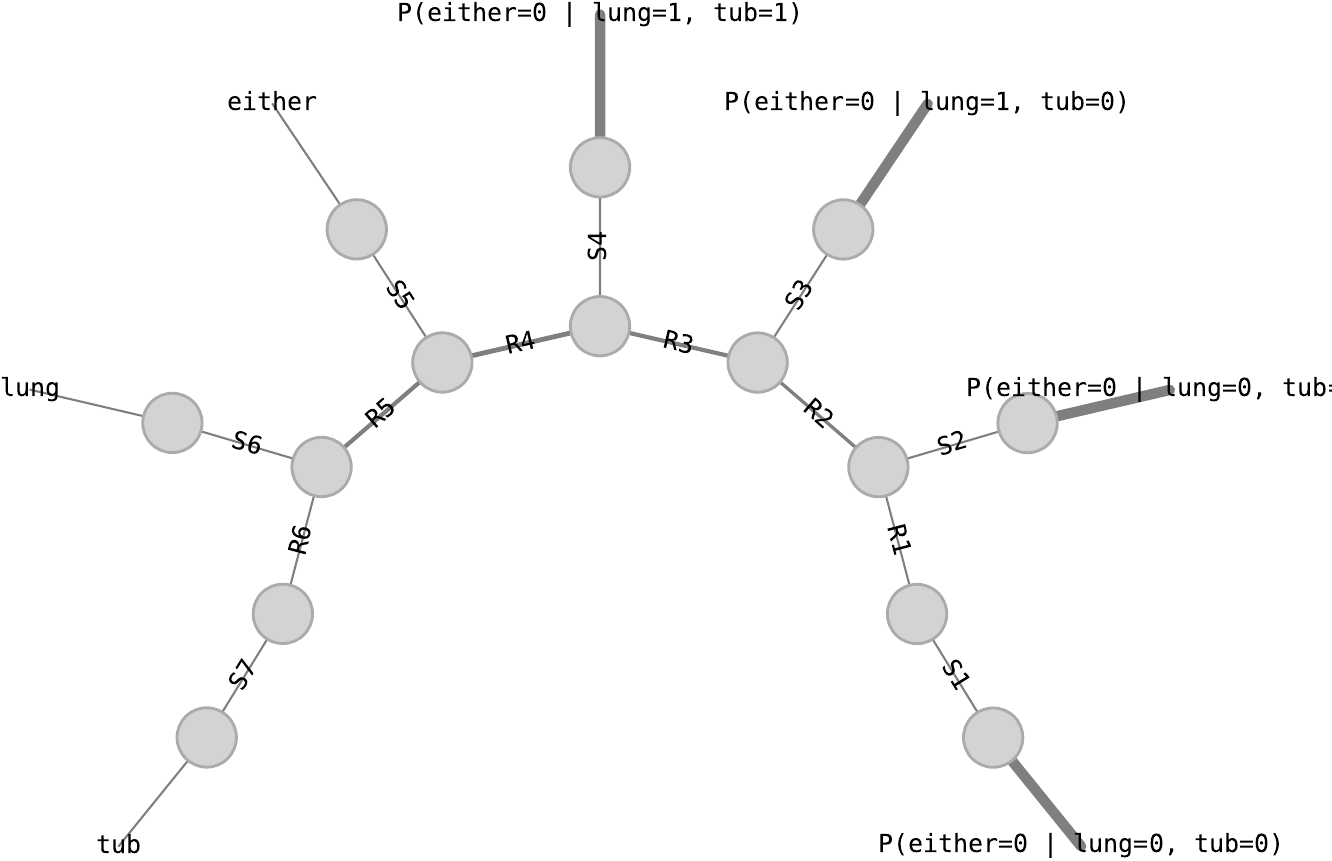}
		\caption{TT representation of $\Phi$ after regarding four of its entries as uncertain and encoding them as additional dimensions (edges in bold).}
	\end{subfigure}
	\caption{Encoding four uncertainties into a 3D CPT results in a new potential of dimensionality $3 + 4 = 7$, which we compactly express using the TT format. The new indices $R_k$ and $S_k$ are \emph{virtual} (in a probabilistic interpretation, they can be thought of as \emph{latent variables}) and only serve to factorize the original tensor into smaller-dimensional tensors, thus avoiding an exponential increase in the number of parameters.}
	\label{fig:bn_mrf_tn}
\end{figure}

\subsection{Computation of the Sobol Indices}

The manipulations described above turn the original TN into a new one with different potentials and topology.
We view the new TN as a convenient factorization of a high-dimensional function, namely the one mapping each possible value of the new variables (uncertain CPT entries) to the target probability of interest. We are now ready to obtain the Sobol indices for these variables. We do so by applying Ballester-Ripoll and Leonelli's algorithm~\citep{ballester2022computing}, which calculates the required variances by further transforming the TN, followed by performing ordinary inference operations on the result. We use variable elimination as our inference method, which makes the Sobol estimation exact, but other inference algorithms could be used instead. 

See Fig.~\ref{fig:asia_modified} for an illustration of the \textsc{asia} network after encoding 10 uncertain parameters.

\begin{figure}\centering
	\begin{subfigure}[b]{0.48\columnwidth}
		\includegraphics[width=1\columnwidth]{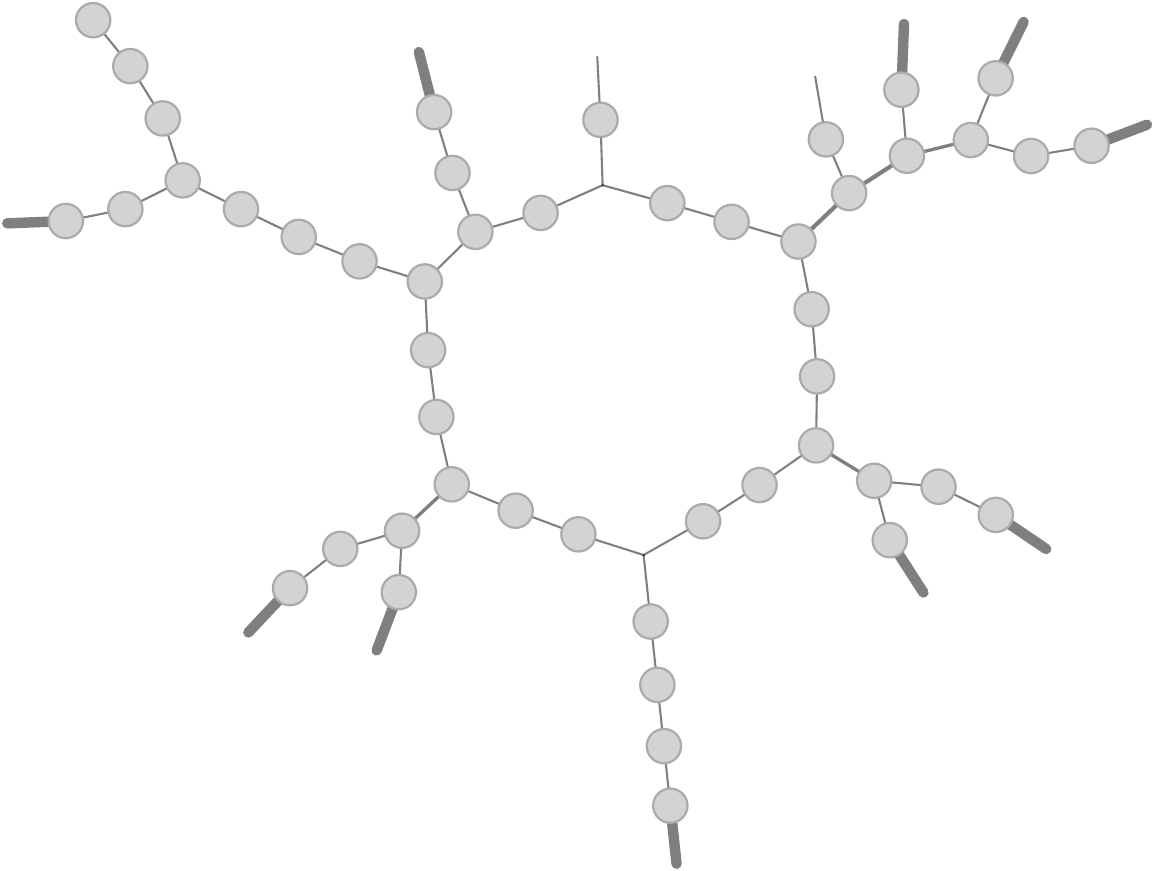}
		\caption{}
        \label{fig:asia_modified}
	\end{subfigure}
    \hfil
	\begin{subfigure}[b]{0.48\columnwidth}
		\includegraphics[width=1\columnwidth]{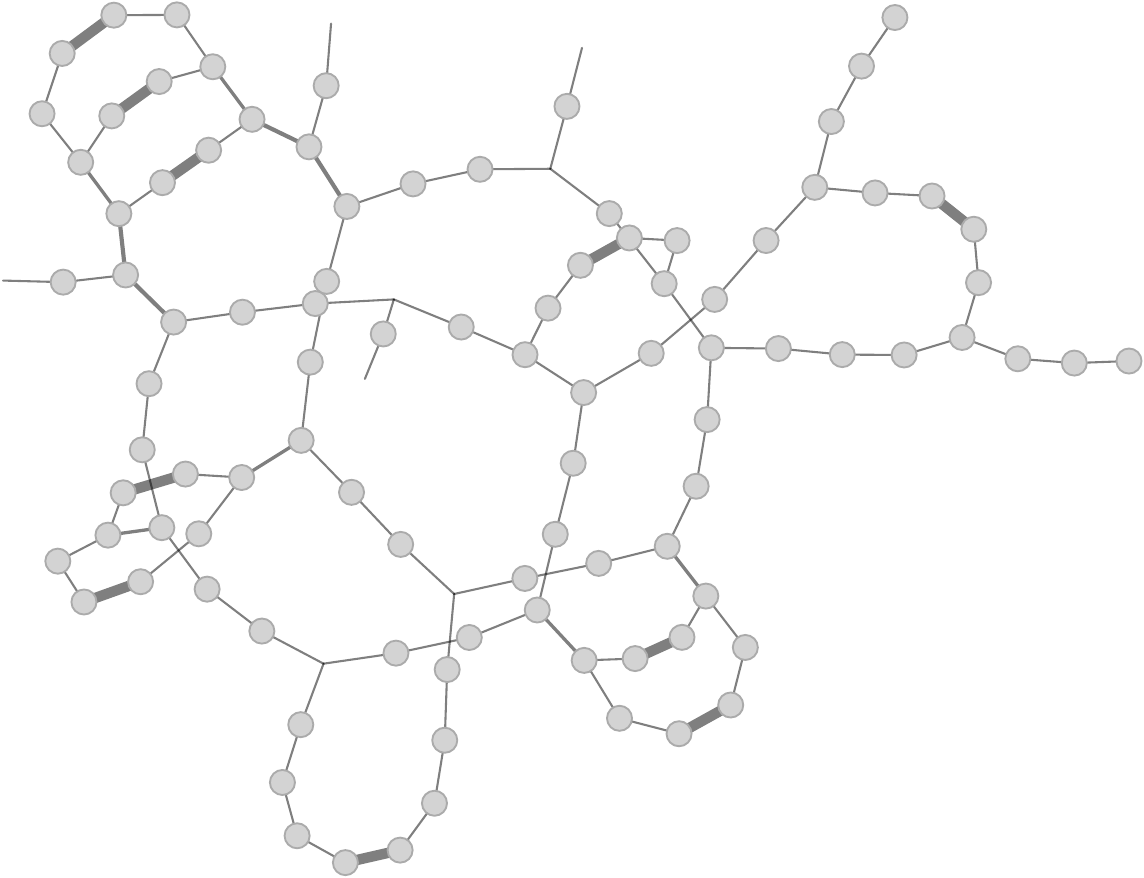}
		\caption{}
        \label{fig:asia_modified_product}
	\end{subfigure}
    \caption{Left: \textsc{asia} network after including 10 uncertain CPT entries as new dimensions (edges in bold). Right: squared network, used to compute the network's variance w.r.t. the 10 new dimensions in the method of Sobol~\cite{ballester2022computing}, built by glueing two copies of the network along the target indices. Index names in both diagrams were omitted to avoid visual cluttering.}
	\label{fig:asia_full}
\end{figure}


\section{Experimental Results}
\label{sec:experiments}

\subsection{Setup}
\label{sec:setup}

Since the proposed method does not support multiple uncertainties for same-parent configurations (Sec.~\ref{sec:encoding_uncertainties}), we cannot consider all BN parameters as uncertain. Instead, in our experiments we use a heuristic based on the sensitivity value (Sec.~\ref{sec:parametric_sensitivity_analysis}) to select potentially sensitive uncertain parameters among $\boldsymbol{\theta}$:

\begin{enumerate}
\item Select a target probability of interest $P(Y_O = y_O)$.
\item Compute the sensitivity value of all parameters of the BN using \citet{ballester2022yodo}.
\item For each CPT of the BN, select the rows where all sensitivity values are non-zero.
\item For each of those rows, mark as uncertain the parameter with the highest sensitivity value.
\item Assume the the uncertain parameters are independently distributed, with each $\theta_i$ following a beta distribution with mean equal to its original value $\theta_i^0$ and fixed variance $\sigma^2$. More explicitly, $\theta_i \sim \text{Beta}(\alpha = \theta_i^0 \lambda, \beta = (1 - \theta_i^0) \lambda)$ where $\lambda := \frac{\theta_i^0 - \sigma^2 - (\theta_i^0)^2}{\sigma^2}$. We will use $\sigma^2 = 0.02$ in our experiments.
\end{enumerate}

Note that, instead of heuristically, the selection of uncertainties and their distributions could also be expert-elicited.

\subsection{Software}

Ultimately, the proposed method relies on efficient inference of large TNs to produce the desired indices. We do this via variable elimination, also known as \emph{tensor contraction} in the TN literature. The order in which variables are eliminated has a massive impact on the overall computational burden, and finding the optimal order is an NP-hard problem. We use the TN contraction library \emph{cotengra}~\citep{graykourtis21} and its heuristic \textsc{auto-hq}, which offers a very good compromise between time spent on order planning and on actual evaluation. For certain tensor manipulation operations we use \emph{quimb}~\citep{gray2018quimb}, and we use \emph{PyTorch} and \emph{YODO} for calculating sensitivity values \citep{ballester2023yodo}.



\subsection{Network}

We next illustrate our algorithms over a BN learned from the survey ``Eurobarometer 93.1: standard Eurobarometer and COVID19 Pandemic" \citep{eu}. Out of the many questions asked in the survey, we selected some demographic information of the respondents together with their opinion about how the COVID-19 emergency was handled by local authorities and its consequences in the long term.  To avoid large number of levels and rare answers, possible categories were combined into meaningful groups. Table \ref{tab} gives details about the final 15 variables considered. Entries with missing values were dropped giving a total of 30985 observations.

\begin{table}
    \centering
    \label{lit1}
    \scalebox{0.53}{
\begin{tabular}{p{5cm}p{11.5cm}p{9.8cm}}
    \toprule
\textbf{Label} & \textbf{Question} & \textbf{Levels} \\
\midrule 
AGE &
How old are you? & 18-30/30-50/51-70/70+
  \\
CLASS & Do you see yourself and your household belonging to\dots? & Working class/Lower class/Middle class/Upper class \\
COMMUNITY & Would you say you live in a\dots & Rural area or village/Small or middle sized town/Large town\\
COPING (CP)& Thinking about the measures taken to fight the Coronavirus outbreak, in particular the confinement measures, would you say that it was an experience\dots? & Easy to cope with/Both easy and difficult to cope with/Difficult to cope with\\
COUNTRYFIN (CF)& The Coronavirus outbreak will have serious economic consequences for your country& Agree/Disagree/Don't know \\
GENDER & What is your sex?& Male/Female \\
HEALTH (HT)& Thinking about the measures taken by the public authorities in your country to fight the Coronavirus and its effects, would you say that they\dots & Focus too much on health/Focus too much on economivcs/Are balanced\\
INFO &Which of the following was your primary source of information during the Coronavirus outbreak? & Television/Written press/Radio/Websites/Social networks\\
JUSTIFIED (JD) &Thinking about the measures taken by the public authorities in your country to fight the Coronavirus and its effects, would you say that they were justfied? & Yes/No\\
LIFESAT (LS)& On the whole, are you satisfied with the life you lead?& No/Yes\\
OCCUPATION & Are you currently working?& Yes/No \\
PERSONALFIN (PF)&The Coronavirus outbreak will have serious economic consequences for you personally & Agree/Disagree/Don't know\\
POLITICS &In political matters people talk of 'the left' and 'the right'. How would you place your views on this scale? & Left/Centre/Right/Don't know\\
SATMEAS (SM)& In general, are you satisfied  with the measures taken to fight the Coronavirus outbreak by your government?& Yes/No \\
TRUST & Do you trust or not the people in your country?& Yes/No\\
\bottomrule
\end{tabular}}
\caption{Details of the variables used to learn the \textsc{covid-measures} BN. \label{tab}}
\end{table}

A BN over this dataset was learned, which we henceforth refer to as \textsc{covid-measures} model, with the tabu algorithm of the \texttt{bnlearn} R package in combination with a non-parametric bootstrap of the data \citep{Scutari2010}. Two-hundred replications of the data were created and for each replication a BN was learned. Edges that appeared in more than 54.5\% of the learned networks were retained in the final model \citep[this percentage was chosen using the method of][]{scutari2013identifying}. The resulting network has 15 nodes, 34 edges
 and 588 parameters (see Fig.~\ref{fig:covid_measures}).

\begin{figure}
    \centering
    \includegraphics[scale=0.24]{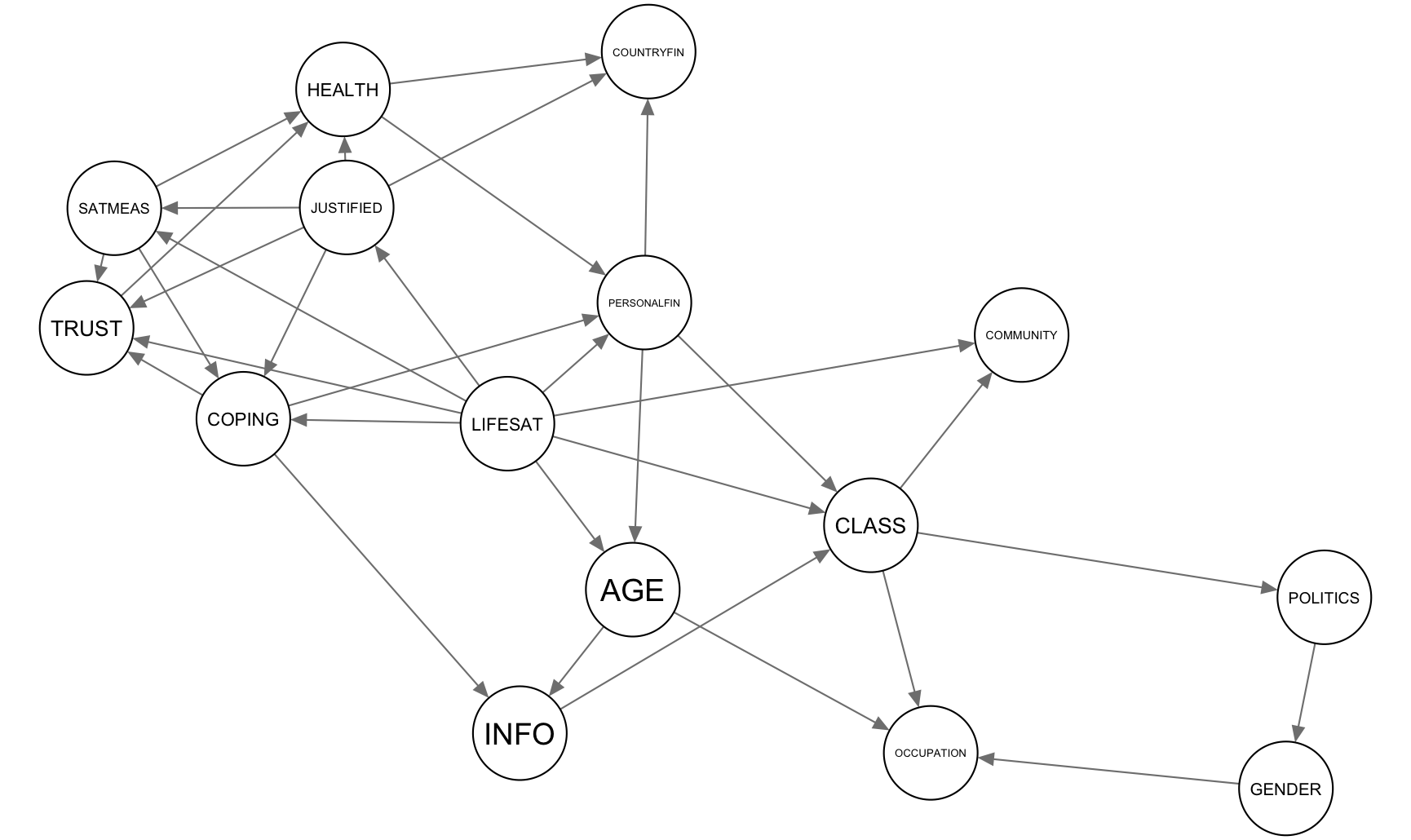}
    \caption{BN learned on the COVID-19 pandemic Eurobarometer dataset.}
    \label{fig:covid_measures}
\end{figure}

\subsection{Analysis}

We target first the probability of interest $P(\text{COUNTRYFIN} = \text{Agree})$ in the \textsc{covid-measures} model, and apply the proposed method to identify the most influential parameters of the BN. After restricting ourselves to the conditions detailed in Sec.~\ref{sec:setup}, we are left with $|\boldsymbol{\theta}| = 55$ uncertainties. Fig.~\ref{fig:covid_measures_modified} shows the resulting tensor network after baking $\boldsymbol{\theta}$ in.
\begin{figure}
    \centering
    \includegraphics[scale=0.6]{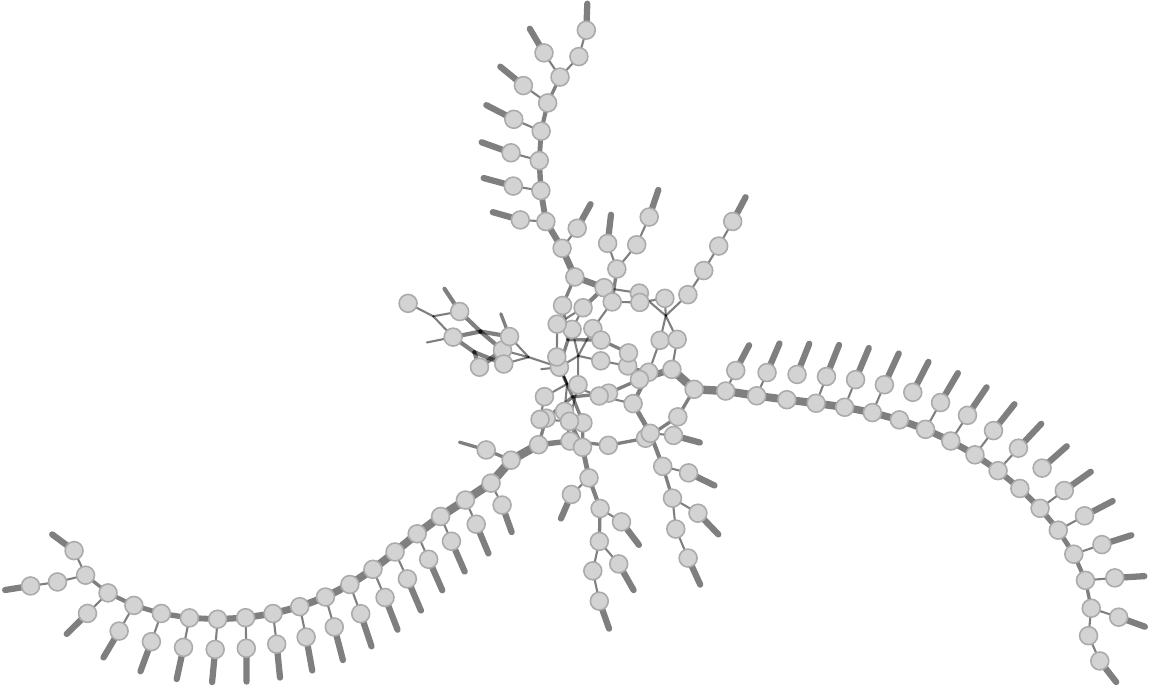}
    \caption{Tensor network for the \textsc{covid-measures} model after encoding 55 uncertain parameters. This visualization uses the Kamada-Kawai graph layout.}
    \label{fig:covid_measures_modified}
\end{figure}

The uncertainties $\boldsymbol{\theta}$ we analyzed include up to 18 entries per individual CPT, such as that of variable PERSONALFIN. After encoding, the resulting node has $3 + 18 = 21$ parents and its enriched CPT would take millions of parameters at any reasonable discretization level. However, thanks to the low-rank factorization we propose, the output tensor network from Fig.~\ref{fig:covid_measures_modified} only has 12'783 parameters in total, comprising tensors up to 6 dimensions and  1'026 elements only.

Fig.~\ref{fig:country_finance_suffer} shows the resulting top 40 BN parameters, sorted by descending total index.
\begin{figure}
    \centering
    \includegraphics[scale=0.6]{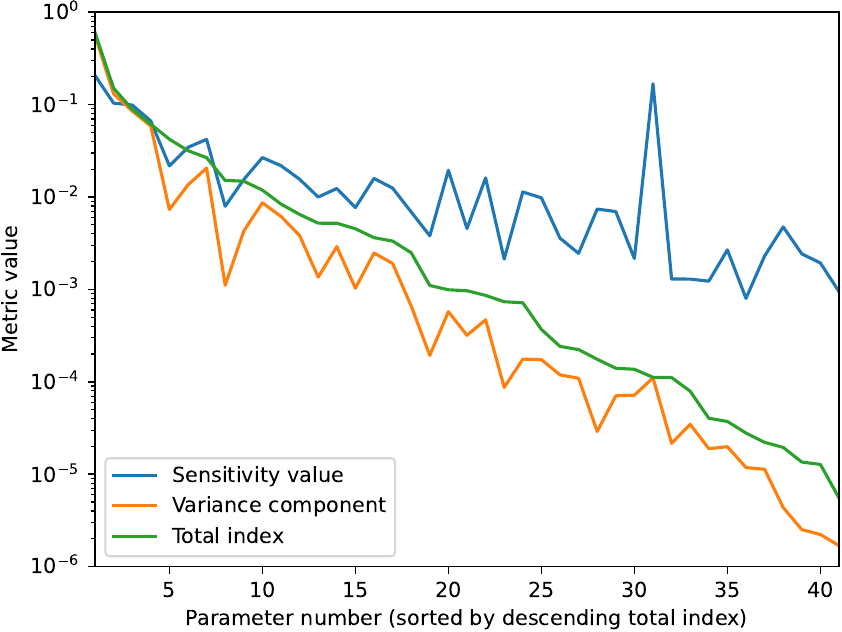}
    \caption{Sensitivity measures for the \textsc{covid-measures} model and target $P(\text{COUNTRYFIN} = \text{Agree})$.}
    \label{fig:country_finance_suffer}
\end{figure}

There are some significant differences between the OAT analysis (the sensitivity value, in blue) and the global approach we propose. For example, the 31$^{st}$ parameter\footnote{Corresponding to $P(\text{COUNTRYFIN} = \text{Agree} | \text{HEALTH} = \text{Are balanced}, \\ \text{JUSTIFIED} = \text{No}, \text{PERSONALFIN} = \text{Agree})$.} has a sensitivity value of $|\partial f / \partial \theta_{31}| = 0.168$ (the second highest), but its Sobol indices are almost insignificant at roughly $10^{-4}$. Conversely, the Sobol indices of the top parameter\footnote{Corresponding to $P(\text{COUNTRYFIN} = \text{Agree} | \text{HEALTH} = \text{Are balanced}, \\ \text{JUSTIFIED} = \text{No}, \text{PERSONALFIN} = \text{Disagree})$.} are three times larger than its sensitivity value $|\partial f / \partial \theta_1| = 0.207$. Looking at the big picture, the Sobol indices drop much faster than the sensitivity values, suggesting that the OAT analysis overestimates the importance of most BN entries. The sensitivity values have Spearman correlations of only $0.507$ and $0.461$ with the variance components and total indices, respectively.

The proposed analysis also sheds light on the \emph{interaction} between uncertain CPT entries. For each parameter $\theta_i$, the difference between the green and orange lines is equal to the interaction term $S_{\theta_i}^T - S_{\theta_i}$, which denotes how much of the parameter's influence is due to joint effects with one or more other uncertain parameters. For example, in the case of the 9$^{\text{th}}$ parameter in Fig.~\ref{fig:country_finance_suffer}~\footnote{Corresponding to $P(\text{COUNTRYFIN} = \text{Don't know } | \text{ HEALTH} = \text{Focus too much on health}, \text{ JUSTIFIED} = \text{No}, \text{ PERSONALFIN} = \text{Don't know})$.}, we have that $S_{\theta_9}^T$ is almost 14 times larger than $S_{\theta_9}$.

See Tab.~\ref{tab:country_finance_suffer} for the full analysis on the top 20 uncertain parameters.
\begin{table}
\renewcommand{\arraystretch}{0.8}
\footnotesize
\begin{center}
\begin{tabular}{ccccc}
\toprule
CPT entry & \specialcell{Original \\ value \\ $\theta_i^0$} & \specialcell{Sensitivity \\ value \\ $|\partial f / \partial \theta_i|$} & \specialcell{Variance \\ component \\ $S_{\theta_i}$} & \specialcell{$\downarrow$ Total \\ index \\ $S_{\theta_i}^T$} \\
\midrule
P(CP=2 \textbar LS=1, SM=0, JD=0) & 0.396000 & 0.001295 & 0.562670 & 0.600831 \\
P(CF=0 \textbar HT=0, JD=0, PF=1) & 0.890525 & 0.103779 & 0.130846 & 0.150597 \\
P(PF=2 \textbar LS=0, HT=1, CP=2) & 0.026115 & 0.009978 & 0.085278 & 0.089189 \\
P(CP=2 \textbar LS=0, SM=1, JD=0) & 0.285907 & 0.002455 & 0.058607 & 0.061405 \\
P(PF=2 \textbar LS=1, HT=2, CP=2) & 0.025797 & 0.006555 & 0.007327 & 0.042235 \\
P(PF=2 \textbar LS=0, HT=0, CP=1) & 0.034884 & 0.009801 & 0.013552 & 0.031758 \\
P(CF=0 \textbar HT=2, JD=0, PF=2) & 0.728324 & 0.011346 & 0.020582 & 0.026658 \\
P(CF=2 \textbar HT=0, JD=1, PF=2) & 0.157303 & 0.002429 & 0.001110 & 0.015136 \\
P(PF=2 \textbar LS=1, HT=2, CP=0) & 0.021277 & 0.003172 & 0.004286 & 0.014840 \\
P(HT=0 \textbar SM=1, JD=0) & 0.335113 & 0.007694 & 0.008659 & 0.011876 \\
P(CF=0 \textbar HT=1, JD=1, PF=1) & 0.644961 & 0.021831 & 0.006169 & 0.008394 \\
P(PF=2 \textbar LS=1, HT=0, CP=1) & 0.035593 & 0.002300 & 0.003833 & 0.006482 \\
P(CF=0 \textbar HT=1, JD=0, PF=1) & 0.845084 & 0.066525 & 0.001359 & 0.005212 \\
P(CF=0 \textbar HT=2, JD=1, PF=2) & 0.633333 & 0.000948 & 0.002905 & 0.005207 \\
P(JD=0 \textbar LS=1, SM=1) & 0.523494 & 0.003805 & 0.001030 & 0.004533 \\
P(CF=0 \textbar HT=2, JD=1, PF=0) & 0.919786 & 0.015900 & 0.002478 & 0.003634 \\
P(CF=1 \textbar HT=2, JD=1, PF=1) & 0.229560 & 0.012489 & 0.001911 & 0.003337 \\
P(CF=0 \textbar HT=0, JD=0, PF=2) & 0.816038 & 0.007410 & 0.000662 & 0.002495 \\
P(HT=0 \textbar SM=1, JD=1) & 0.530692 & 0.015577 & 0.000193 & 0.001103 \\
P(CF=0 \textbar HT=1, JD=0, PF=2) & 0.712500 & 0.004749 & 0.000574 & 0.000991 \\
\bottomrule
\end{tabular}
\end{center}
\caption{Top 20 measures for the \textsc{covid-measures} model and target $f = P(\text{COUNTRYFIN} = \text{Agree})$. 
}
\label{tab:country_finance_suffer}
\end{table}

A second example with the same model, this time for target probability $P(\text{INFO} = \text{Radio})$, is shown in Fig.~\ref{fig:country_finance_info}. The 8$^{\text{th}}$ parameter, $P(\text{INFO} = \text{Radio} | \text{Age} = 18-30)$, has a sensitivity value of $0.161$, apparently indicating that the effectiveness of communicating COVID-19 lockdown measures via radio channels is largely driven by the young segment. However, its Sobol indices are both below $0.0015$, revealing the parameter to be virtually irrelevant.
\begin{figure}
    \centering
    \includegraphics[scale=0.6]{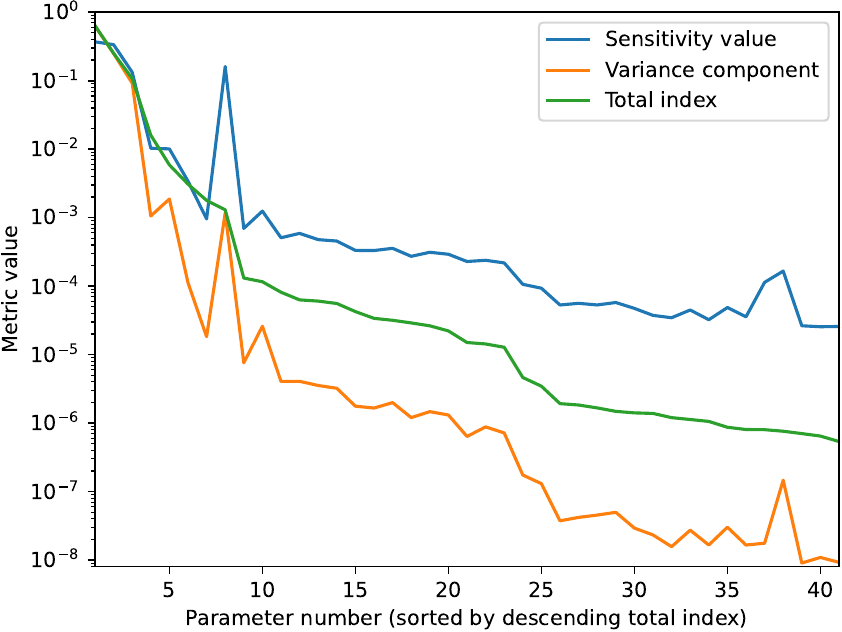}
    \caption{Sensitivity measures for the \textsc{covid-measures} model and target $P(\text{INFO} = \text{Radio})$. Note the blue spike denoting a large sensitivity value, whereas the Sobol indices are below 0.0015, confirming that the parameter is not relevant.}
    \label{fig:country_finance_info}
\end{figure}

\subsection{Comparison Across Networks}

Last, we conduct a comparative study that applies the proposed method to the BN above as well as 7 other networks from the literature. Results are reported in Tab.~\ref{tab:bn_comparison} and include metrics about the network before and after encoding uncertain parameters, the resulting Sobol indices, and computation timings.

\begin{table}
    \small
    \scalebox{0.8}{
    \begin{tabular}{lllllllll}
\toprule
 & alarm & crisis & child & \specialcell{covid- \\ measures} & hailfinder & hepar2 & insurance & sachs \\
\midrule
N. variables & $37$ & $21$ & $20$ & $15$ & $56$ & $70$ & $27$ & $11$ \\
Treewidth & $5$ & $6$ & $7$ & $8$ & $16$ & $17$ & $9$ & $6$ \\
N. parameters & $752$ & $183$ & $344$ & $588$ & $3741$ & $2139$ & $1419$ & $267$ \\
N. uncertainties & $7.2$ & $17.6$ & $25.0$ & $76.2$ & $8.2$ & $27.6$ & $32.2$ & $16.2$ \\
N. augmented parameters & $1622.2$ & $974.4$ & $3156.4$ & $1.25 \cdot 10^{5}$ & $4662.2$ & $4318.2$ & $7480.4$ & $2296.4$ \\
Max. $S^T - S$ & $0.022$ & $0.022$ & $0.02$ & $0.064$ & $0.084$ & $0.15$ & $0.018$ & $0.032$ \\
Max. $|S^T - f'|$ & $0.24$ & $0.27$ & $0.26$ & $0.21$ & $0.19$ & $0.45$ & $0.15$ & $0.13$ \\
$t_{\text{encoding}}$ (s) & $0.11$ & $0.2$ & $0.59$ & $5.69$ & $0.11$ & $0.51$ & $0.72$ & $0.29$ \\
$t_{\text{Sobol}}$ (s) & $0.29$ & $0.38$ & $0.8$ & $7.67$ & $0.24$ & $2.58$ & $1.99$ & $0.44$ \\
$t_{\text{total}}$ (s) & $0.39$ & $0.58$ & $1.39$ & $13.36$ & $0.35$ & $3.1$ & $2.71$ & $0.74$ \\
\bottomrule
\end{tabular}

    }
    \centering
    \caption{Comparison across multiple BNs, including the newly learned \textsc{covid-measures}. Each columns was generated as the average of 5 independent runs: at each run, the target variables $Y_O$ and its target level $y_o$ were chosen at random. The network \textsc{crisis} is the BN proposed in~\cite{ballester2022yodo}; the remaining ones come from the \texttt{bnlearn} repository~\citet{Scutari2010}.  \label{tab:bn_comparison}}
\end{table}

Tab.~\ref{tab:bn_comparison} shows that interaction terms can get as large as $S^T - S \approx 0.15$ in the examples tested, while the difference between a total index and its corresponding sensitivity value went as high as $|S^T - f'| \approx 0.45$. The table also gives insights on the typical computing times required to apply the proposed method to compute all first-order Sobol indices. It can be noticed that the encoding requires less time than computing the actual Sobol indices. In all cases the total time increases with the number of augmented parameters, which of course depend on the number of variables and the treewidth. We can see that in all cases the Sobol indices can be computed in a short amount of time, less than a second in 50\% of the BNs considered.

\section{Conclusions}
\label{sec:conclusions}

Our work identifies the limitations of OAT-based sensitivity analysis in BNs and proposes a global, variance-based alternative that can account for multiple uncertain CPT parameters varying at once. The algorithm is based on embedding the uncertain probabilities as additional parents of the affected CPTs, and compactly encoding these augmented tables via low-rank factorization. The factorization of each CPT can be expressed as a TN, which means the resulting network is also a TN and is therefore amenable to tensor contraction (equivalent to variable elimination). Thanks to this TN formulation, we were able to apply a Sobol index computation algorithm~\citep{ballester2022computing} that relies on element-wise product between TNs, followed by exact tensor contraction. We demonstrate the method with a BN fitted on a real-world survey; we find cases where either the classical sensitivity value fails to account for interactions between parameters, or conversely it vastly overestimates the influence of individual parameters. The algorithm runs in a few seconds at most, even in networks with dozens of nodes and thousands of parameters. All in all, we argue the method allows analysts to conduct a broader and more informed sensitivity analysis on BNs subject to parameter uncertainty.

Despite the above, the proposed method is not free of limitations. The clearest one is its inability to model uncertainties for multiple child state probabilities given the same configuration of parents. This is because the sum-to-1 constraint introduces dependencies between those CPT entries, and the Sobol method is only suitable for non-correlated parameters of interest. However, note that this is not a problem for binary variables, since there the probabilities for alternate child states are fully linked to each other and therefore will lead to identical sensitivity indices.

\subsection*{Future Work}

We have restricted ourselves to an exact tensor contraction scheme, as it was sufficient to compute Sobol indices from a TN that, after augmentation, comprised 50-100 nodes. For larger or topologically more complex networks, note that the TN framework opens the door to \emph{approximate} contraction, which leverages SVD rank truncation to compress intermediate factors on the fly during inference. In exchange for a modest error, approximate schemes are better suited to tackle inference in challenging TNs~\citep{graychan24}.

As another venue of future research, upon learning which uncertain entries are the most influential in a global SA sense, one could then further obtain higher-order Sobol indices and derivations thereof in order to discern which interactions between entries are important. Options include the \emph{closed} indices, the \emph{lower} and \emph{upper} indices, or the \emph{Shapley values}~\citep{owen:14}, among other metrics.

\bibliographystyle{elsarticle-harv} 
\bibliography{references}

\end{document}